\documentclass{article}
\usepackage{ijcai20}

\renewenvironment{thebibliography}[1]{%
\begin{oldthebibliography}{#1}%
    \setlength{\parskip}{0\baselineskip}
    \setlength{\itemsep}{0\baselineskip}
}%
{%
\end{oldthebibliography}%
}

\usepackage{amsthm,amsmath,amssymb}
\usepackage{times}
\usepackage{soul}
\usepackage{url}
\usepackage[hidelinks]{hyperref}
\usepackage[utf8]{inputenc}
\usepackage[small]{caption}
\usepackage{graphicx}

\usepackage{booktabs}
\usepackage{algorithm}
\usepackage{color}
\usepackage[noend]{algpseudocode}
\newtheorem{thm}{\bf Theorem}

\theoremstyle{definition}
\newtheorem{defn}{Definition}[section]

\title{Automatic Hyper-Parameter Optimization Based on Mapping Discovery from Data to Hyper-Parameters}

\author{
Bozhou Chen$^1$
\and
Kaixin Zhang$^1$\and
Longshen Ou$^1$\and
Chenmin Ba$^1$\and
Hongzhi Wang$^1$\And
Chunnan Wang$^1$
\affiliations
$^1$Harbin Institute of Technology\\
\emails
bozhouchen@hit.edu.cn,
\{1170300216,1170300321,1170300524\}@stu.hit.edu.cn,
\{wangzh,WangChunnan\}@hit.edu.cn
}

\begin{document}
\maketitle
\begin{abstract}
Machine learning algorithms have made remarkable achievements in the field of artificial intelligence. However, most machine learning algorithms are sensitive to the hyper-parameters. Manually optimizing the hyper-parameters is a common method of hyper-parameter tuning. However, it is costly and empirically dependent. Automatic hyper-parameter optimization (autoHPO) is favored due to its effectiveness. However, current autoHPO methods are usually only effective for a certain type of problems, and the time cost is high. In this paper, we propose an efficient automatic parameter optimization approach, which is based on the mapping from data to the corresponding hyper-parameters. To describe such mapping, we propose a sophisticated network structure. To obtain such mapping, we develop effective network constrution algorithms. We also design strategy to optimize the result futher during the application of the mapping. Extensive experimental results demonstrate that the proposed approaches outperform the state-of-the-art apporaches significantly.

\end{abstract}
\section{Introduction}
Automatic machine learning (autoML) have gained wide attention and applications in both industry and academia. Automatic hyper-parameters optimization is one of the most critical parts. The effectiveness of many machine learning algorithms is extremely sensitive to parameters\cite{zkx_ref_1}. Without a good set of hyper-parameters, the machine task cannot be solved well even with optimal model.

Among the hyper-parameter optimization approaches, data-driven methods draw attentions since they could achieve effective prediction of hyper-parameters based on historical experience implicit in the data.

However, data-driven automatic hyper-parameter optimization faces three severe challenges. Firstly, exisitng systems may involve thousands of machine learning tasks with many hyper-parameters\cite{zkx_ref_2}. Recalculating hyper-parameters for each task may cause large time cost. Thus, the optimization process should be efficient. Secondly, the hyper-parameter optimization algorithm should have good transferability. The reason is that the optimal hyper-parameters are always different for different datasets, without transferability, the hyper-parameter optimization algorithm has to be run many times even for the same machine learning algorithm. Thirdly, a hyper-parameter optimization algorithm should be able to handle many parameters, since some complex machine learning algorithms have thousands of hyper-parameters\cite{zkx_ref_3}, which are required to be optimized to ensure the effectiveness.

Even though some hyper-parameter optimization algorithms have been proposed. They could not solve all these problem. 	
MI-SMBO\cite{feurer2015initializing} optimizes hyper-parameters based on historical data's meta-feature. It accelerates the optimization process and significantly improves the algorithm’s performance. However, in this approach, the meta-features are selected manually, which limited its transferability. Also, as a kind of SMBO module, since MI-SMBO needs to run the machine learning algorithm iteratively for many times, the low efficiency of MI-SMBO may be caused by inefficient machine learning algorithm. \cite{Rijn2018hyper-parametersIA} selects the hyper-parameters with the most significant influence to the performance, and predicts priors hyper-parameters based on the best $n$ group hyper-parameters in the historical datasets. This method also needs to iterate the machine learning algorithm and could hardly optimize complex algorithms in limited time.

This motivate us to solve these problems. Intuitively, the optimal parameters are determined by two factors, i.e., the machine learning algorithm and the data. Therefore, under the same algorithm, the parameters are completely determined by the data. Thus, we attempt to investigate the relationship between parameters and data. Considering that each dataset corresponds to at least one set of optimal hyper-parameters, we believe that there is a mapping from data space to parameter space, and describe this mapping with a neural network. As a result, we use this mapping to achieve prediction of hyper-parameters directly.

Our contributions of this paper are summarized as follows.

\begin{list}{\labelitemi}{\leftmargin=1em}\itemsep 0pt \parskip 0pt
\item We consider the mapping from data to the optimal hyper-parameters and apply this mapping to the selection of the optimal hyper-parameters. On different tasks of an algorithm, the model has strong transferability, which greatly saves time overhead. For this reason, the model can achieve ultra-high-dimensional optimization of hyper-parameters.

\item With XGBOOST as an example, we design the neural network structure for the mapping as well as traing approaches, which could be applied to other machine learning tasks with slight modificaiton.

\item Experimental results on real data demonstrate that the proposed approach significantly outperforms the state-of-art algorithms in both accuracy and efficiency.
\end{list}

In the remaining of this paper,  Section~\ref{sec:method} describes the proposed approach. Extensive experiments are conducted in Section~\ref{sec:exp}. We overview related work in Section~\ref{sec:related}. Section~\ref{sec:con} draws the conclusions.

\section{Method}

\label{sec:method}
The basic idea of our approach is to build a mapping from datasets to the optimal hyper-parameters and use this mapping to take parameter determination according to a given dataset. Since mapping is the core concept of our approach, we define it at first and overview the algorithm in Section~\ref{sec:overview}, and then discuss the major components of our algorithm in Section~\ref{sec:building} and Section~\ref{sec:application}, respectively.

\subsection{Overview}
\label{sec:overview}

For a machine learning algorithm, the optimal hyper-parameters should be specific for the dataset $D$. From this aspect, optimal hyper-parameters generator for a algorithm could be considered as training a mapping from a dataset $D$ to an optimal parameter vector $P$, which is defined as follows.

	\begin{defn}
			\textbf{Parameter Mapping}: For a machine learning algorithm $ALGO$, the mapping from each training dataset $D$ of $ALGO$ to the corresponding optimal hyper-parameter vector $P$ is called a parameter mapping from $D$ to $P$ w.r.t $ALGO$, determined by $MAP_{ALGO}$
\label{def:mapping}
		\end{defn}
	Since the mapping catches complex features of the data and may be very complex, we attempt to use a neural network to represent this mapping, which is called a \emph{\underline{C}ore \underline{N}etwork} (CN). Thus, our algorithm is divided into two phase, CN construction and CN application as is shown in Figure~\ref{fig:system}. And they are described in Section~\ref{sec:building} and Section~\ref{sec:application}, respectively. Before them, we introduce the structure of CN at first.
				\begin{figure}
					\setlength{\abovecaptionskip}{0.cm}
					\setlength{\belowcaptionskip}{-0.cm}
					\centering
					\includegraphics[scale = 0.2]{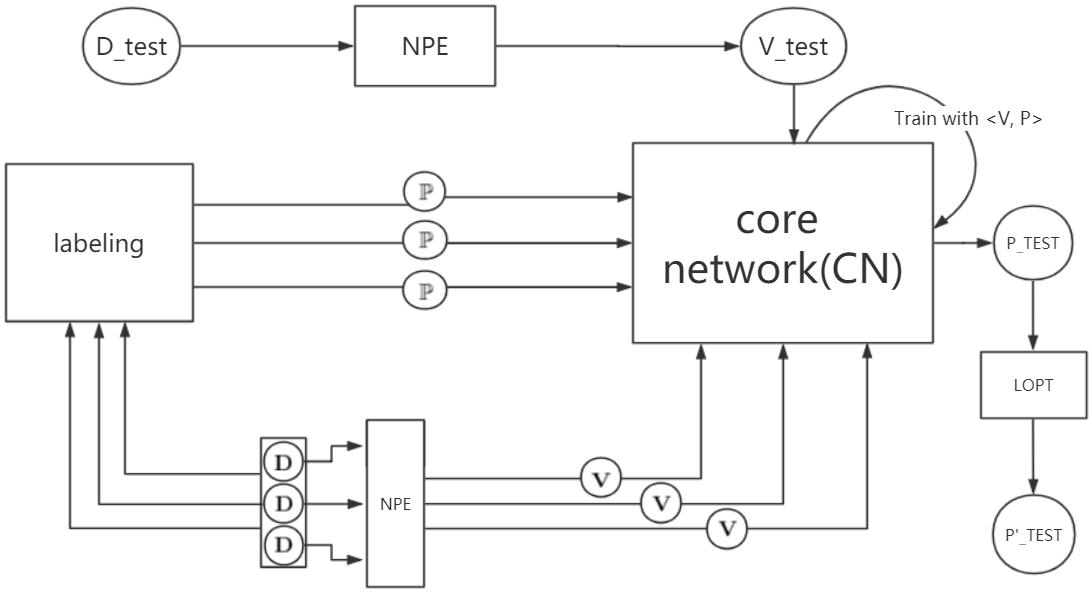}
\label{fig:system}
					\caption{The components of the proposed algorithm is shown in this figure. The CN construction process of is shown in the lower left, and the rest shows the CN application component for hyper-parameter prediction.}
					\label{fig:system}
				\end{figure}


	\subsection{CN Structure}
To build the mapping from the meta-feature to optimal hyper-parameters, we train a neural network called CN for each algorithm. The input of the CN is meta-feature of the dataset, and the output is the generated hyper-parameters. Clearly different machine learning task corresponds to different hyper-parameters and require different CN correspondingly. In this section, we
introduced the CN for XGBOOST\cite{zkx_ref_4}. The CNs for other tasks could be constructed in the similar way.

			\begin{figure}[t]
					\setlength{\abovecaptionskip}{0.cm}
					\setlength{\belowcaptionskip}{-0.cm}
				\centering
				\includegraphics[scale = 0.15]{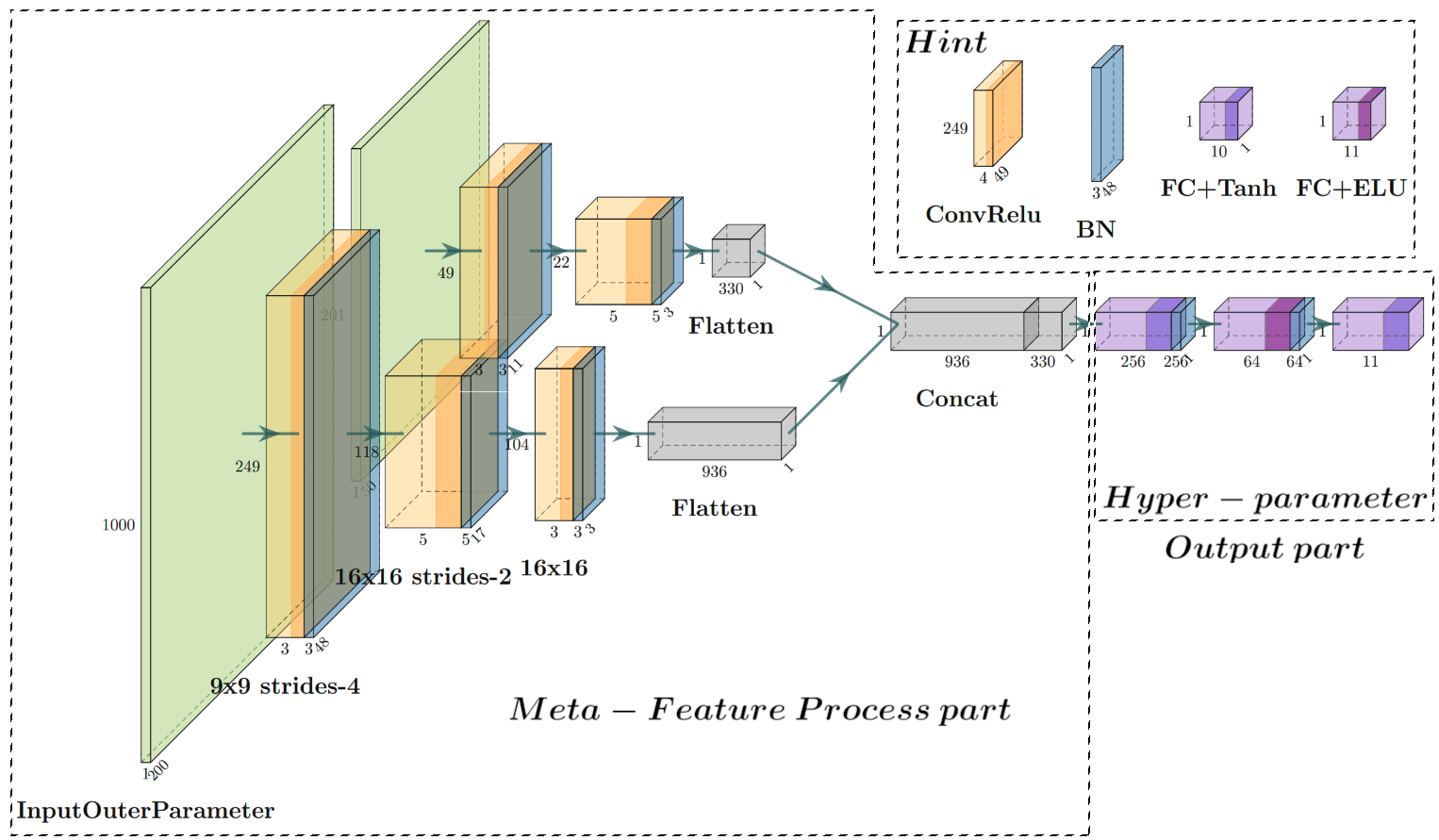}
				\caption{An example of core network.}
				\label{CoreNet}
			\end{figure}

The CN's structure is shown in Figure~\ref{CoreNet}. For such CN, we attempt to model the dataset as a neutral network and use its trainable parameters as the input. More specifically, the parameters refers to neutral network's weights and bias. The former is represented as a 2-dimension matrix, and the latter is a float.  The CN has two part, i.e., Meta-Feature Process part and Hyper-parameter Output part.

Meta-Feature Process part analyze the trainable parameter and reduce its size. In order to fully retain the structural information of the meta-features, we concatenate the biases of the same layer to weights, and then use these 2-dimension matrices as the CN input. Since these matrices are often large, to reduce the difficulty of CN's training, we use a convolution layer (ConvRelu component in Figure~\ref{CoreNet}) to reduce the size of the inputs.  According to the output format, the output of each convolution layer is flattened (Flatten component in Figure~\ref{CoreNet}) and concatenated (Concat component in Figure~\ref{CoreNet}).

Hyper-parameter Output part combines the result of Meta-Feature Process part and predict the hyper-parameters. In this part, the fully connected layer (FC+Tanh and FC+ELU components in Figure~\ref{CoreNet}) uses active functions such as Tanh and ELU\cite{inproceedings} to make the range of output within that of the hyper-parameter. The activate function's selection depends on the range and distribution of the hyper-parameter\cite{akak}.


To construct CN for a machine learning algorithm $A$, we first select suitable neural network structure based on the data type of the input of $A$. For example, for image classification, the neural network could be CNN\cite{6795724}. While for NLP, the structure could be LSTM\cite{LSTM}.
After this, the inputs of CN, i.e., the trainable parameters of the neural network, are determined by the dimensions and data types of trainable parameters. Next, to construct Meta-Feature Process part, we use convolution layers similar as the CN described above to reduce the size of the input. Finally, we flatten the result of the inputs, and concatenate them into a 1-dimension vector if the results have multiple branches. Hyper-Parameter Output part is constructed by several fully-connected layers, and we need to adjust the last layer's output to fit the number of hyper-parameters. We can also use Batch-Normalization\cite{DBLP:journals/corr/IoffeS15} to increase the effectiveness of learning after the convolution layers if necessary.

	\subsection{CN construction}
\label{sec:building}

To construct CN, we have two major jobs. The first is to prepare suitable data for the CN training, and the second is to train the CN. We introduce them in this section, respectively.

\subsubsection{Data Preparation}
The preparation of data has two goals, i.e., sufficiency and task-fitting. To achieve the former, we develop sampling technique to generate sufficient training data from original datasets which are suitable for this problem and contain sufficient data. As for the latter, we propose encoding approach to extract the meta-features of the generated dataset, as is the input of CN. Additionally, we need to label the generated dataset by generating the corresponding optimal hyper-parameters.


 \underline{Sampling} To obtain sufficient amount of the training data, we attempt to sample them from the a large dataset, which could be the union of training datasets. Note that to increase the generalization ability of the CN, the training datasets should be diverse. Due to a large amount of standard training dataset published online, it is easy to obtain such dataset. For example, for XGBOOST, we could easily obtain 98 datasets for classification from UCI datasets \footnote{\url{https://archive.ics.uci.edu/ml/datasets.php}}.

Clearly, to ensure the generalization ability of the trained CN, the sampled training data should be dissimilar. We measure the similarity of two datasets $S_i$ and $S_j$ with Jaccard similarity, i.e., $sim(S_i,S_j)$=$\frac{|S_i\cap S_j|}{|S_i\cup S_j|}$. If $sim(S_i,S_j)$$>\delta$, where $\delta$ is a threshold and $0 < \delta < 1$, and a small $\delta$ means a strong constraint for independency, then $S_i$ and $S_j$ are \emph{similar}.
Suppose we perform $m$ samplings to obtain $\mathbb{S}$=\{$S_1$, $\cdots$ $S_{k'}$\}. If for any $S_i$ and $S_j$ in $\mathbb{S}$, $sim(S_i,S_j)$$<\delta$, $\mathbb{S}$ is \emph{independent} w.r.t $\delta$. Clearly, the independence of $\mathbb{S}$ ensures the dissimilarity among the sampled training set. Fortunately, according to Theorem~\ref{thm:sampling}, $\mathbb{S}$ obtained by randomly sampling is independent.

\begin{thm}
\label{thm:sampling}
For a dataset $D$, the minimum size of sample set $k$ and $S<|D|$, there exists $m$ and a reasonable $\delta$ such that a sample set $\mathbb{S}$=\{$S_1$, $\cdots$, $S_{k'}$\} randomly sampled from $D$ with $|S_i|=S$ for any $0<i\leq k'$ is independent w.r.t $\delta$ and m. Mean while $E[k'] \geq k$.
			\end{thm}
			\begin{proof}{sketch}
we only need to ensure the $expectation\ of\ k' \geq k$ to ensure that the sample set is unbiased. It is easy to satisfy by taking out all the subsets of $m$ samples  independent to other samples (suppose the number is $ n $). As long as the instance with $ k '= n $ and $ E [n] \ geq k $ has a solution, we can get the result such that $ E [k'] \ geq k $. The proof of $ E [n] \ geq k $ is straightforward, since $ E [n] $ = $ m \ cdot p_0 ^ {(m-1)} $, which has a maximum value. We only need to prove that the maximum value of the function is greater than $ k $, and this matter can be solved by a differentiation directly.
			\end{proof}
				Now we take MNIST dataset as an example. Suppose $N = 60000$, $S = 1000$, ${\delta}_0 = 0.2$.
				When $N$ and $S$ take the above values, $p_0$ is very close to 1, and have $p_0 > 0.99999$.
				If $p_0 = 0.99999$, $E(n) \geq 1000$, then $1800 \leq m \leq 2000$. So sampling at least 1800 times can ensure that at least 1000 subsets are independent to each other. So it's no problem to get enough independent subsets.
				\begin{figure}
					\setlength{\abovecaptionskip}{0.cm}
					\setlength{\belowcaptionskip}{-0.cm}
					\centering
					\includegraphics[scale = 0.3]{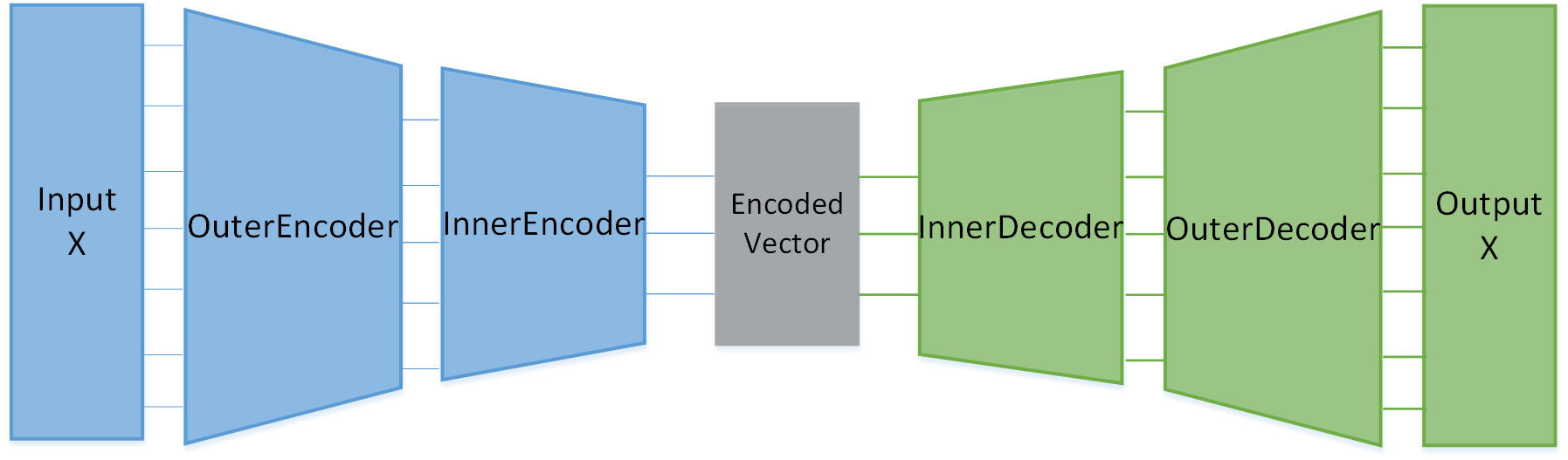}
					\caption{Normal-NPE}
					\label{encode1}
				\end{figure}
				\begin{figure}
					\setlength{\abovecaptionskip}{0.cm}
					\setlength{\belowcaptionskip}{-0.cm}
					\centering
					\includegraphics[scale = 0.2]{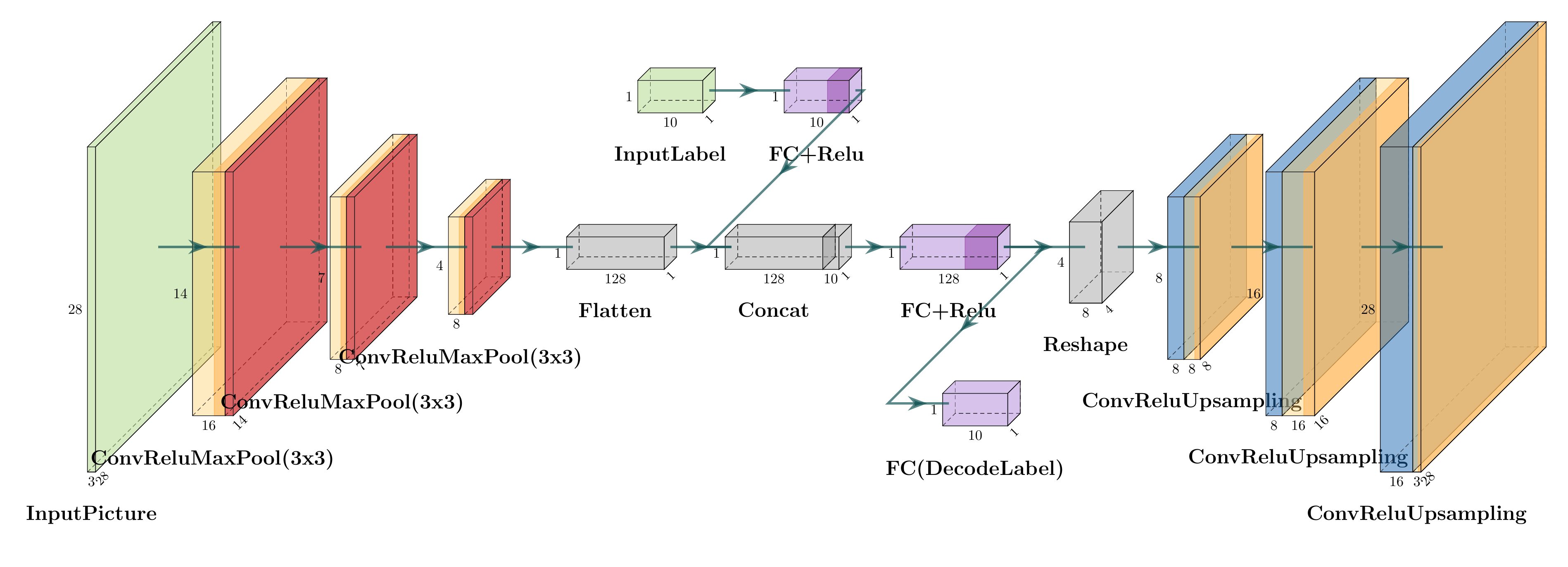}
					\caption{Image-NPE.}
					\label{encode2}
				\end{figure}

 \underline{Encoding} Even though the dataset may be various, the input of CN should be uniformed by encoding. Two issues must be addressed here. One is that the feature numbers of the dataset may be different. The other is that the number of samples in the dataset may be different.

We solve the first one through zero-padding. That is, adding features to a dataset with small number of features. All these new features could be simply set to 0. Thus, the meaning of new dataset is consistent with that before it is processed. Experimental results have shown that zero-padding does not affect key indicators such as classification accuracy.

For the second issue, we design  \emph{\underline{N}etwork \underline{P}arameter \underline{E}mbedding}(NPE) approach, which uses an auto-encoder to encode the dataset, and returns the neural network parameters of the encoder as the feature of the dataset. There are two differences between NPE and traditional auto-encoders.
On the one hand, NPE encodes attributes and the label at the same time, because only when the attributes and label are jointly encoded, the result can represent the features of the original dataset. On the other hand, as discussed above, we use the parameters of corresponding neural network representing the dataset as the input of CN. Therefore, in the application phase, each dataset is encoded to parameters.

Since different dataset may be represented by different neural network with different parameters. We develop two types of NPEs, Normal-NPE and Image-NPE to fit to typical types of datasets for our CN.
In this paper, we focus on these two data types and will study NPEs for more data types in the future.

	Table-NPE is used to process table data, i.e., each sample is a one-dimensional vector. The input and output of the encoder are one-dimensional. Considering that the number of features of the data is not particularly large, it can use the fully connected layer as the main structure of the network. Here we use a stack autoencoder\cite{zkx_ref_9}, whose structure is shown in Figure~$\ref{encode1}$. It can effectively encode structured datasets.
		
	The essence of Image-NPE is a convolutional auto-encoder, which is designed to encode unstructured datasets such as images. Firstly, a convolutional self-encoder commonly used in image encoding problems is used. However, considering that the label of the dataset should be encoded at the same time as the picture, it can be ensured the encoding result can represent the original dataset. Therefore, at the output layer of the encoder, the encoding result of the picture is flattened (saves the picture structure information), and the label is jointly encoded by the fully connected layer, and then the fully linked layer is used to separate the label and the picture in the decoder part. Finally, the reshape layer is used to recover the picture using the previously saved picture structure information, which is decoded and output by the convolutional decoder. The structure of Image-NPE is shown in Figure~$\ref{encode2}$.

 \underline{labeling} The goal of labeling is to generate the optimal hyper-parameters as the label for the sampled dataset $S_i$ to form the input of CN with the encoded features of $S_i$. Our solution is to compute the hyper-parameters with existing approaches such as the work in \cite{Frazier2018ATO} and pick the best one according to the experiments.


\subsubsection{CN Training}

Intuitively, CN can be trained with normal neural network training approaches such as the work in \cite{zkx_ref_10}. The major challenge in CN training is that the loss is too large and could not converge.

The first cause is that the label range of the training data is too wide to learn.
To solve this problem, we first zoom the labels with tanh, then adjust the activation function of neural network's output layer to tanh, so that the output's range can fit the labels' perfectly.

The second cause is the gradient explosion\cite{zkx_ref_12}. To handle this issue, the gradient clipping was performed, and the full link layer activation function is changed to tanh. 
The third cause is that the prediction result of the CN for a label with a large value is small and has little change during training. This is due to the saturation of tanh. To solve this problem, the data with the larger value in the label is taken as log10 before tanh function computation.
After the application of these strategies, the CN's loss in the training gets small, and the validation sets can converge steadily.

Besides, in order to improve the generalization, dropout technology could be used.\cite{ckck}

\subsection{CN Application}
\label{sec:application}
After CN is trained, it can be applied to generate the optimal hyper-parameters. Before prediction, it is still necessary to encode the dataset to import the dataset into CN. This process is exactly the same as that in CN construction.

Note that although the CN prediction results may still contain some errors, which may cause a huge loss in performance. Therefore, we need to optimize the output of CN furthermore. Since in most of the cases, even with the errors, the results generated with CN are around the optimal results\cite{zkx_ref_15}, we attempt optimize the parameters within this local area, as is called local optimization.

 \underline{Local Optimization}
	Suppose $\boldsymbol{P}$ is the output of CN. We divide $\boldsymbol{P}$ into two subsets $\boldsymbol{P}_l$ and $\boldsymbol{P}_r$, and then divide these two subsets recursively until one subset contains just one or two parameters. If a subset has just one parameter, then the function MC(mountaining climbing method) is invoked. If a subset contains two parameters, then the function DMC(dual mountain climbing method) is invoked.

The whole process is shown in the Algorithm $\ref{lopt}$. LOPT$'s$ input is the output of CN, and its output is the optimized parameters. In this algorithm, we put all the parameters to optimize in a list $P$. Similar as Quick-sort, during recursion, we optimize the parameters in a range in $P$.

Line 1 and 2 initializes a segment tree and invoke FUNC method for local optimization, which is in Line 3-11. FUNC has three parameters. P is the parameter list. l and r is the range of parameters in $P$ to optimize in this function. In each loop in FUNC, the parameters in $P$ is optimized recursively in Line 8-11 until converge. In this algorithm, to acclerate the judgement of converge, we maintain the sum of the absolute update values of parameters in a segment tree\cite{zkx_ref_16}, such that for each loop, the sum is unnecessary to be recomputed, and the total time of converge checking is changed from $O(n)$ to $O(logn)$, where $n$ is the parameter number. Line 12-30 is the process of MC method, which is the mountain climbing process and updates the segment tree with the updated value. Line 31 to 34 is the process of DMC method, which optimize $P[l]$ and $P[r]$ iteratively.
			\begin{figure}[t]
					\setlength{\abovecaptionskip}{0.cm}
					\setlength{\belowcaptionskip}{-0.cm}
				\centering
				\includegraphics[scale = 0.3]{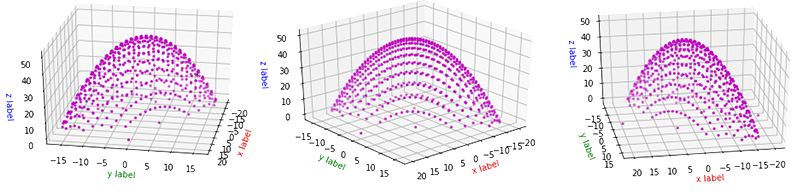}
				\caption{The distribution of the algorithm's accuracy when changing its two parameters. Test$\_$ENV is the algorithm to be adjusted.}
				\label{distr}
			\end{figure}
				\begin{algorithm}[h]
\scriptsize
				  \caption{LOPT}
				\label{lopt}
				  \centering
				  \begin{algorithmic}[1]
				    \Require $\boldsymbol{P}$
				    \Ensure $\boldsymbol{{P}'}$
					\State $seg\_tree \gets init()$
					\State $\boldsymbol{{P}'} \gets \Call {func}{\boldsymbol{P},1,n}$
					\Function{func}{$\boldsymbol{P},l,r$}
				      \If {$l = r$}
				      \State $\boldsymbol{P} \gets \Call {mc}{\boldsymbol{P},l}$
					\Return $\boldsymbol{P}$
				      \EndIf
				      \If {$l = r - 1$}
				      \State $\boldsymbol{P} \gets \Call {dmc}{\boldsymbol{P},l,r}$
					\Return $\boldsymbol{P}$
				      \EndIf
				    \While {$check\_over(l,r) = false$}
				      \State $mid \gets (l + r) / 2$
				      \State $\boldsymbol{P} \gets \Call {func}{\boldsymbol{P},l, mid}$
				      \State $\boldsymbol{P} \gets \Call {func}{\boldsymbol{P},mid + 1, r}$
				    \EndWhile
					\Return $\boldsymbol{P}$
					\EndFunction
		
					\Function{mc}{$\boldsymbol{P},x$}
					\State $stride \gets \epsilon$, $value\_x \gets \boldsymbol{P}[x]$, $value\_last\_x \gets inf$, $s\_value\_x$
				      \While {$|value\_x - value\_last\_x| > {\epsilon}'$}
					\State$a \gets value\_x - stride$
					\State$b \gets value\_x$
					\State$c \gets value\_x + stride$
					\State$value\_last\_x \gets value\_x$
					\State$\boldsymbol{P}[x] \gets a,accu\_a \gets accu(\boldsymbol{P})$
					\State$\boldsymbol{P}[x] \gets b,accu\_b \gets accu(\boldsymbol{P})$
					\State$\boldsymbol{P}[x] \gets c,accu\_c \gets accu(\boldsymbol{P})$
					\If {$accu\_a > accu\_b$}
					\State $\boldsymbol{P}[x] \gets a$
					\ElsIf{$accu\_c > accu\_b$}
					\State $\boldsymbol{P}[x] \gets c$
					\Else
					\State $stride \gets stride / 2$
					\State $value\_x \gets \boldsymbol{P}[x]$
					\EndIf
					\EndWhile
					\State $\Delta \gets |value\_x - s\_value\_x|$
					\State $\Call {update}{seg\_tree,x,\Delta}$
					\Return $\boldsymbol{P}$
					\EndFunction
		
					\Function{dmc}{$\boldsymbol{P},l,r$}
				    \While {$check\_over(l,r) = false$}
					\State $\boldsymbol{P} \gets \Call{mc}{\boldsymbol{P},l}$
					\State $\boldsymbol{P} \gets \Call{mc}{\boldsymbol{P},r}$
				    \EndWhile
					\Return $\boldsymbol{P}$
					\EndFunction
				  \end{algorithmic}
				\end{algorithm}


		Since the running time of the function MC has a constant upper bound, the time complexity of MC is $O(1)$. We show that the time complexity of Algorithm\ref{lopt} can reach $O(nlogn)$  in Theorem~\ref{the:thm2}.


			\begin{thm}\label{the:thm2}
			Algorithm $\ref{lopt}'s$ time complexity can reach $O(nlogn)$, where $n$ is the parameter number.
			\end{thm}
\begin{proof}(sketch)
The recursive equations based for Algorithm 1 is $T (n) = C T (n / 2)$, $T (1) = 1 + log (n)$, where $ C $ is the number of recursive calls ($ C \ geq 2 $). Solving the equation, $ T (n) = n ^ {logC} log (n) $. When $ C = 2 $, the algorithm has the optimal efficiency of $ O (nlogn) $. 

Now we attempt to prove that $C=2$ could ensure the correctness of the algorithm. Inside FUNC, the parameters are split into two halves, and FUNC is called recursively. Transforming between these two parameter sets is similar to the coordinate rotation transform in the DMC algorithm. After the transformation, it is ensured the parameters adjusted for each iteration are optimized along the optimal path. This ensures that the DMC algorithm will converge to the optimal solution after two iterations. These two parameter subsets divided by the FUNC function are regarded as two parameters. Similar transformations are performed on these two parameters. FUNC recursively calls itself twice and exits the loop directly. It can guarantee that the final result is optimal. 
\end{proof}

\section{Experiments}
\label{sec:exp}
	In this section, we study our approach experimentally with two typical machine learning algorithm, $XGBOOST$ and $CNN$. For the experiment, we collected a series of relevant data suitable for the two algorithms respectively.

\begin{figure}[t]
					\setlength{\abovecaptionskip}{0.cm}
					\setlength{\belowcaptionskip}{-0.cm}
		\begin{minipage}{0.32 \linewidth}
			\centering
			\includegraphics[width=0.95\linewidth]{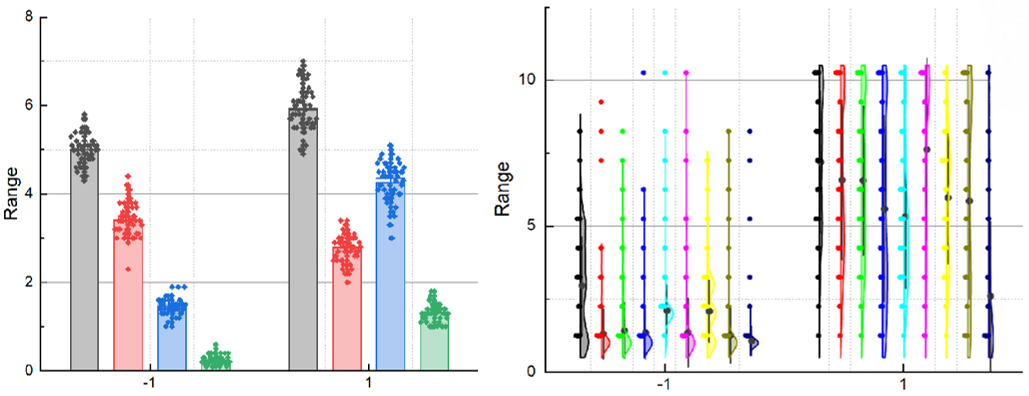}
			\caption{Distribution of features corresponding to different categories.}
			\label{ds_1}
		\end{minipage}
\hfill
	\begin{minipage}{0.32 \linewidth}
			\centering
			\includegraphics[width=0.95\linewidth]{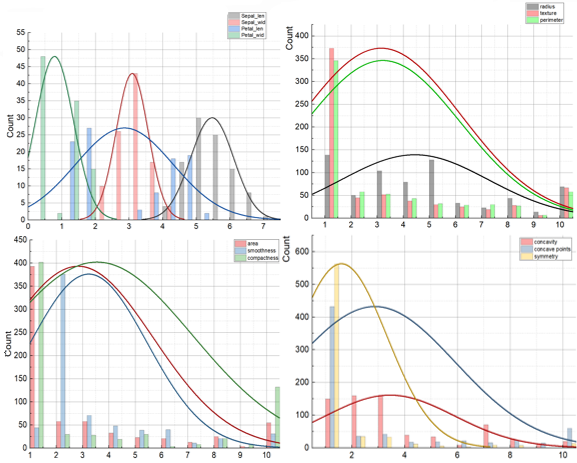}
			\caption{The number of values for each feature.}
			\label{ds_2}
		\end{minipage}
\hfill
			\begin{minipage}{0.32 \linewidth}
				\centering
				\includegraphics[width=0.95\linewidth]{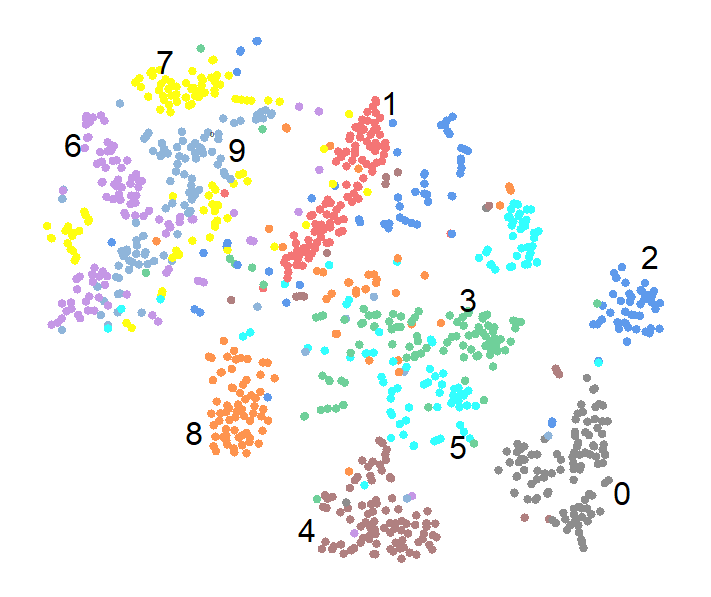}
				\caption{Visualization of MNIST dataset by T-SNE}
				\label{minist}
			\end{minipage}
\end{figure}

	\subsection{Experiment Data}
	\label{sec:ed}

 \underline{XGBOOST} 	We download 98 classification datasets from website\footnote{\url{https://archive.ics.uci.edu/ml/datasets.php}}. Then we sample on those datasets. Here we select two from the datasets for visucalization, as is shown in Figure \ref{ds_1} and Figure \ref{ds_2}.\par
	
 \underline{CNN} We choose MNIST and SVHN datasets, performing random sampling on them according to labels. For subset of MNIST, the size is set to 1000 (100 samples per class), 500 subsets in all. For subset of SVNH, the size is 5000 (500 samples per class), 500 subsets in all.
		In this part's labeling, we use the state-of-the-art derivative free optimization method SRACOS~\footnote{\url{https://github.com/eyounx/ZOOpt}}.

	\subsection{Experimental Settings}
		The software and hardware settings are shown in Table~\ref{tabs:envsetting}, and the data information is in Section~\ref{sec:ed}.
		\begin{table}[t]
					\setlength{\abovecaptionskip}{0.cm}
					\setlength{\belowcaptionskip}{-0.cm}
		\centering
\scriptsize
		\begin{tabular}{l|p{3cm}|p{3cm}}
		\hline
		  & XGBOOST & CNN \\
		\hline
		CPU       & AMD Ryzen 3600&Intel Xeon Platinum 8163 2.5GHz x4 (96 core)\\
		RAM       & 16x2G 3200MHz  & 251G\\
		GPU       & GTX 1060 6G (2000MHz)  & GTX 2080Ti x8\\
		\hline
		OS&Windows10 1903&Linux version 3.10.0-1062.9.1.e17.x86\_64\\
		python&3.7&3.6\\
		keras&2.3.1&2.3.1\\
		tensorflow&1.13.1&1.13.1\\
		tensorflow-gpu&1.13.1&1.13.1\\
		numpy&1.17.4&1.18.0\\
		pandas&0.23.4&0.25.3\\
		scikit-learn&0.22&0.22\\
		others&XGBoost 0.90, Bayesian-optimization 1.0.1&ZOOpt 0.18.2\\
		\hline
		\end{tabular}
		\caption{Software and hardware settings}
		\label{tabs:envsetting}
		\end{table}

	\subsection{Experimental Results}
	We design three groups experiments: blank control group(BCG) without pre-training, control group and experiment group. In the control group, we use Bayesian and ZOOpt to optimize those hyper-parameters. In the experiment group, we use CN and CN + LOPT (or CN) to optimize the algorithm. $\mathcal{P}$,$\boldsymbol{P}'$ and $\boldsymbol{L}$ are their outputs respectively, and they are all a set containing predicted hyper-parameters. Then the three sets of hyper-parameters will be tested on the algorithm(ENV). The overall process is shown in Figure~\ref{experiment} and detailed description is as follows.

		We divide the datasets $\mathcal{D}$ into $\mathcal{X}$ and $\mathbb{X}$ with size 9:1, and $\mathcal{X}$ is used to train CN. Then we continue to divide $\mathbb{X}$ by 9:1 into $\mathbb{X}_i\_{train}$ and $\mathbb{X}_i\_{test}$. We use $\mathbb{X}_i\_{train}$ to run on the three models and obtain the output, then we use  $\mathbb{X}_i\_{test}$ to test the output.

			\begin{figure}[t]
					\setlength{\abovecaptionskip}{0.cm}
					\setlength{\belowcaptionskip}{-0.cm}
				\centering
				\includegraphics[scale = 0.2]{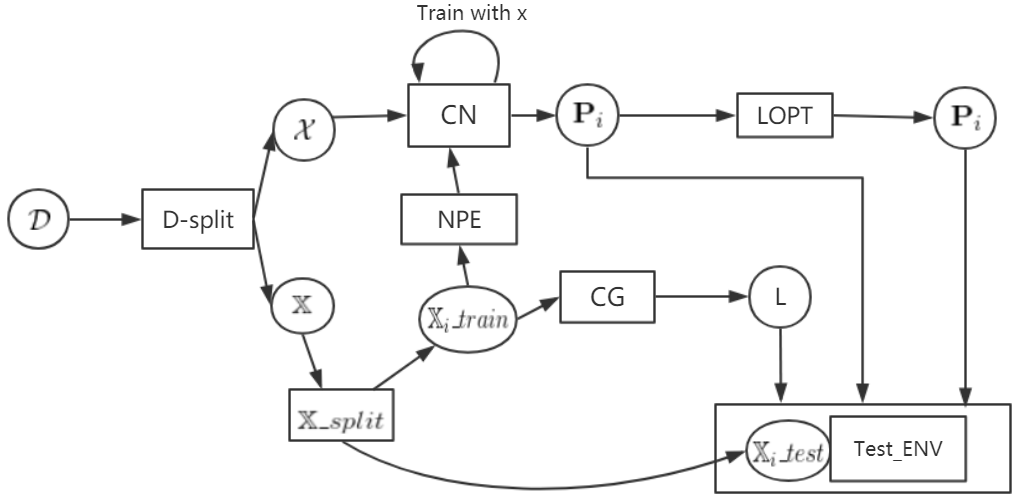}
				\caption{The overall process of our experiment. CN denotes control group. Test\_ENV denotes the algorithm to be adjusted.}
				\label{experiment}
			\end{figure}

 	\underline{XGBOOST}
			We verified the effectiveness of our model (CN, CN + LOPT) on 280 classified datasets, and used the Bayesian optimization algorithm (BO) as a control group, while setting up a blank control group (BCG). The optimal output by the four models are tested, and the accuracy is shown in Figure~\ref{xg_acc}. The horizontal axis represents different test files, and the vertical axis represents accuracy.

			To analyze the pros and cons of each model more clearly, we extract the median, mean, standard deviation, maximum, and quartile of the accuracy corresponding to each of the 280 runs of the tuning algorithm. Statistics are shown in the Table~\ref{tabs:xg}, and visualized in the Figure~\ref{xg_acc_box}. First, according to the performance of BCG, we observe that our dataset has a strong detection ability, which further illustrates that the CN and CN + LOPT models are effective. Second, from the comparison results of CN and CN + LOPT with BO, our model outperforms BO on various indicators. This shows hat our model has strong generalization and migration capabilities.

			To get a deeper understanding of the results on 280 files. For each file, we find the increment of CN + LOPT accuracy relative to BO accuracy. We count the number of files in each incremental interval and draw a pie chart as shown in Figure~\ref{xg_bing}. It can be seen from the results that our model is better than the control group BO on 3/4 of the test data.

			Additionally, we compare the time overhead of the three algorithms (CN, CN + LOPT, BO) for predicting the hyper-parameters of 280 test sets, as shown in Figure~\ref{xg_time}. The horizontal axis still represents different test files, and the vertical axis represents the running time. The time is in log scale. According to the results, our model CN and CN + LOPT outperforms the control group BO. Especially the model CN, which is not optimized locally, will accomplish the task of finding hyper-parameters in just a few seconds. This demonstrates the efficeincy of our model. Because of this, it's possible for our model to optimize algorithms with ultra-high dimensional parameters.

\begin{figure*}
					\setlength{\abovecaptionskip}{0.cm}
					\setlength{\belowcaptionskip}{-0.cm}
			\begin{minipage}{0.32 \linewidth}
				\centering
				\includegraphics[width=0.95\linewidth]{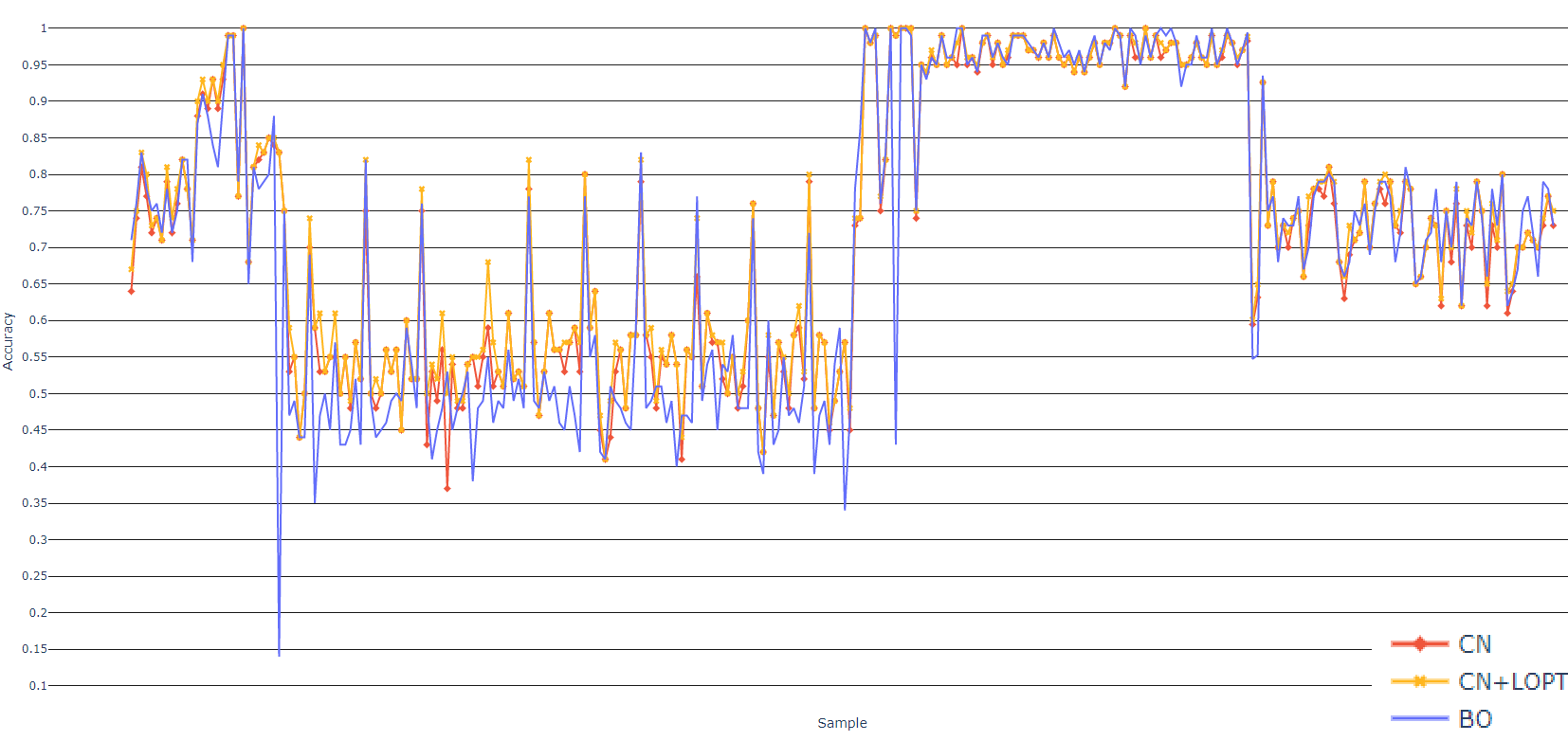}
				\caption{Accuracy for XGBOOST}
				\label{xg_acc}
			\end{minipage}
\hfill
			\begin{minipage}{0.32 \linewidth}
				\centering
				\includegraphics[width=0.95\linewidth]{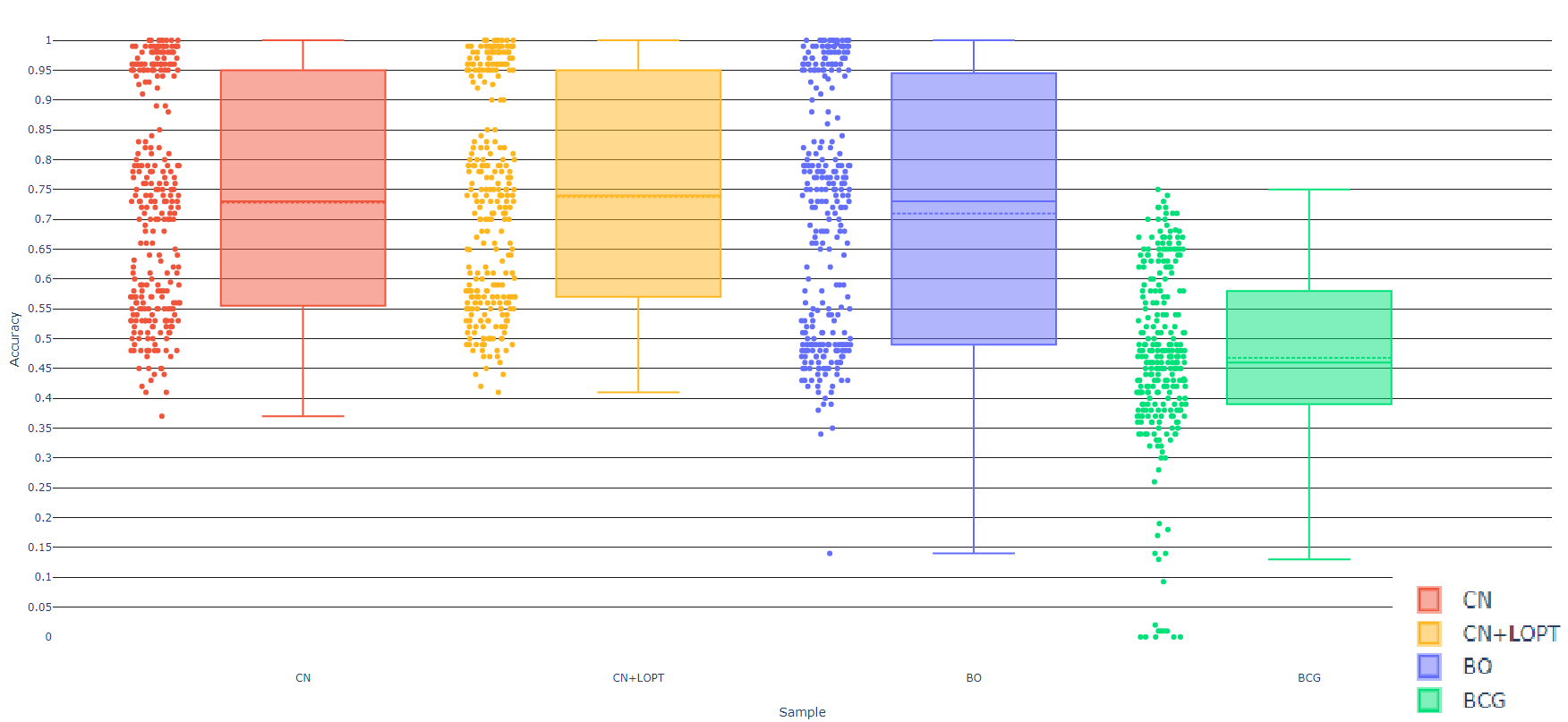}
				\caption{Statistics for XGBOOST}
				\label{xg_acc_box}
			\end{minipage}
\hfill
			\begin{minipage}{0.32 \linewidth}
				\centering
				\includegraphics[width=0.95\linewidth]{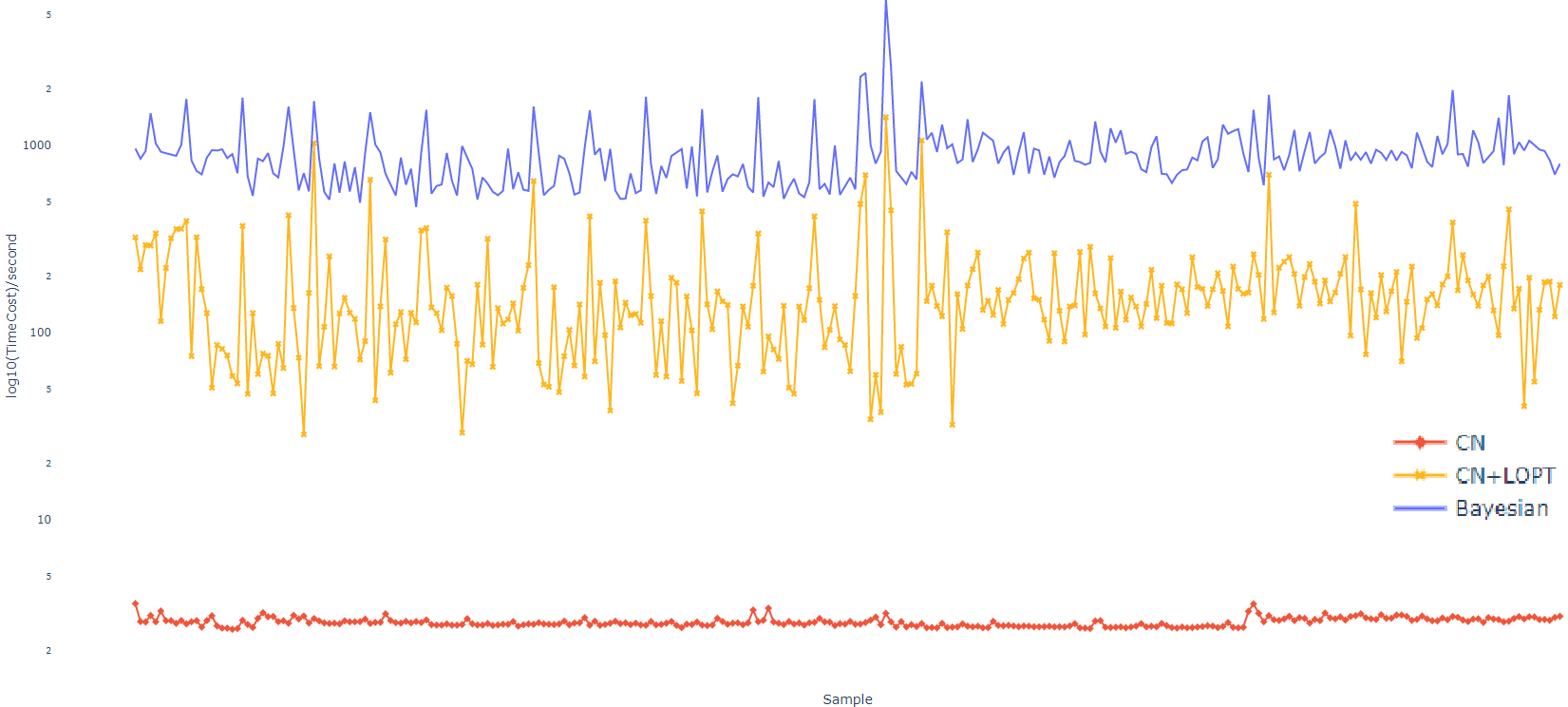}
				\caption{Time overhead comparison.(XGBOOST)}
				\label{xg_time}
			\end{minipage}
\end{figure*}
\begin{figure*}
					\setlength{\abovecaptionskip}{0.cm}
					\setlength{\belowcaptionskip}{-0.cm}
			\begin{minipage}{0.32 \linewidth}
				\centering
			\includegraphics[width=0.95\linewidth]{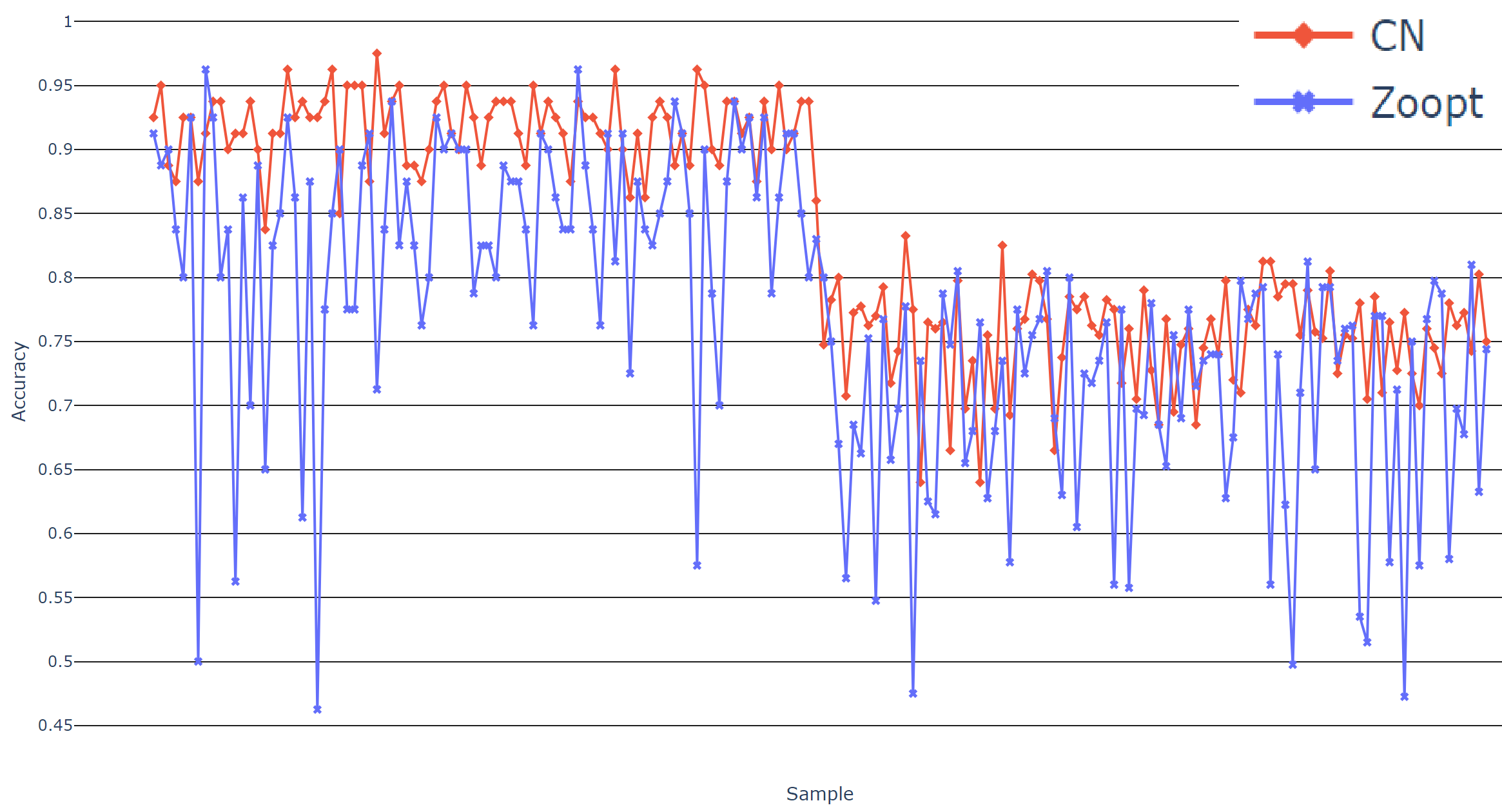}
				\caption{Accuracy for CNN}
				\label{cnn_acc}
			\end{minipage}
\hfill
			\begin{minipage}{0.32 \linewidth}
				\centering
				\includegraphics[width=0.95\linewidth]{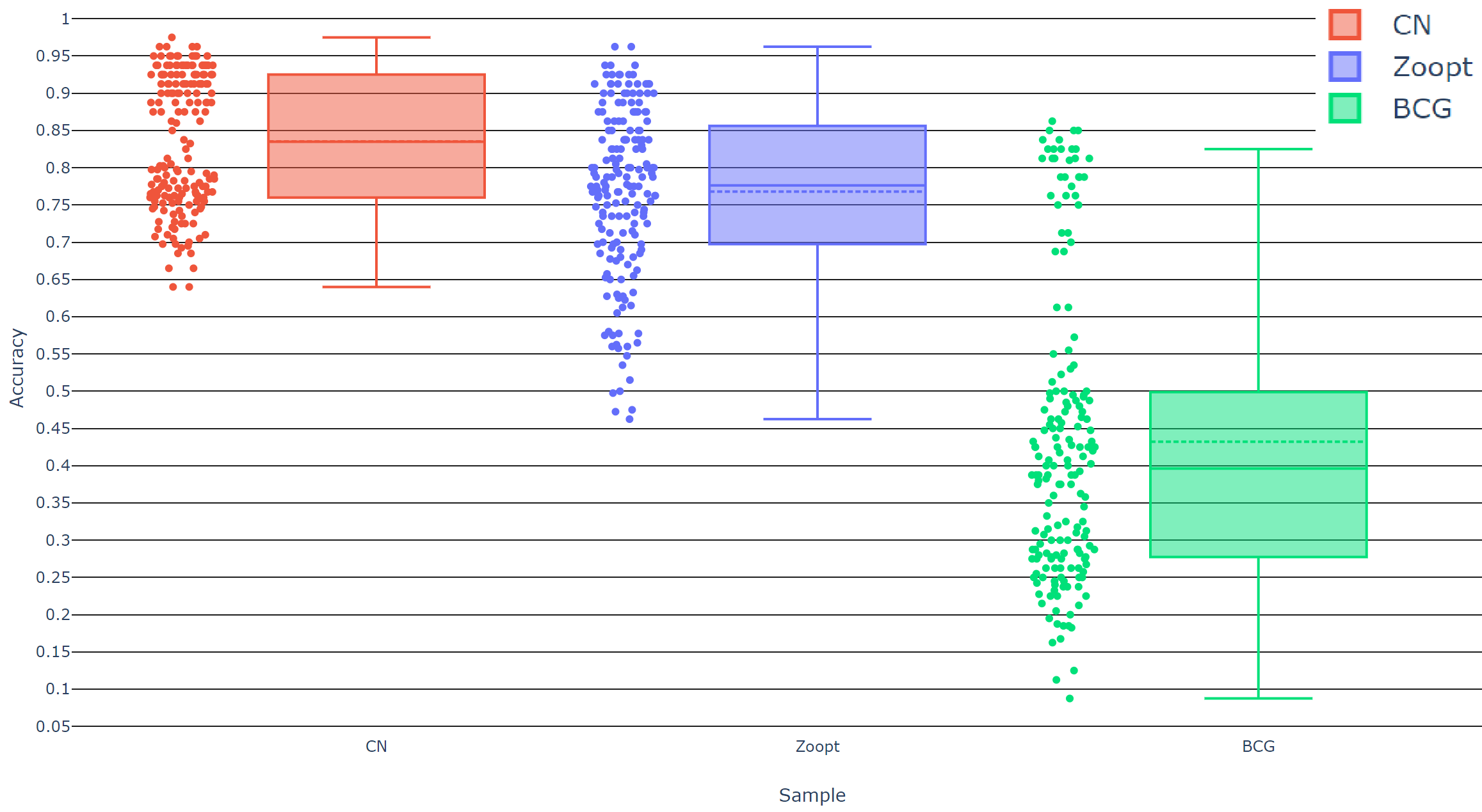}
				\caption{Statistics for CNN}
				\label{cnn_acc_box}
			\end{minipage}
\hfill
			\begin{minipage}{0.32 \linewidth}
				\centering
				\includegraphics[width=0.95\linewidth]{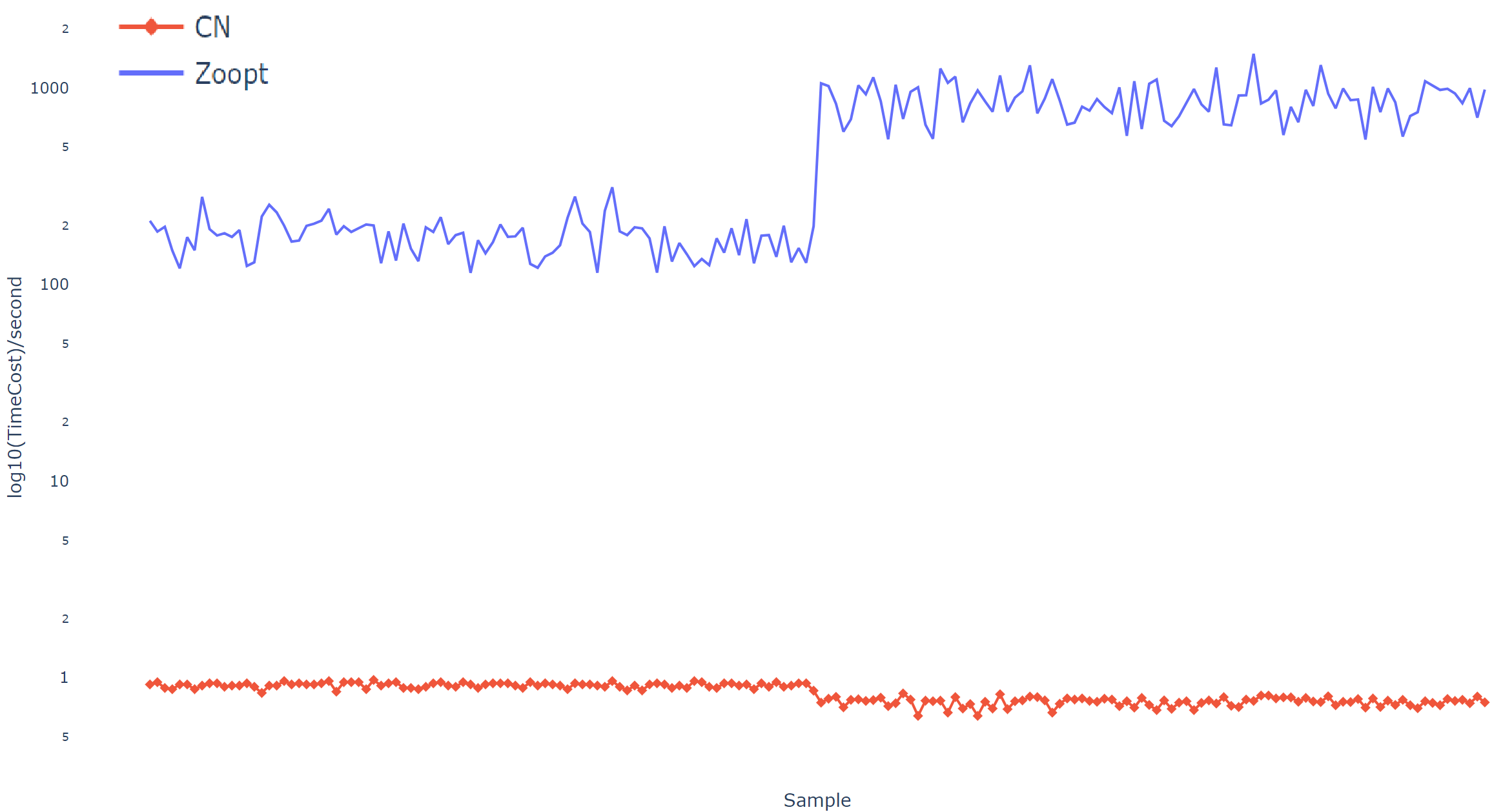}
				\caption{Run Time for CNN}
				\label{cnn_time}
			\end{minipage}
\end{figure*}
			\begin{figure}[t]
				\centering
				\includegraphics[scale = 1]{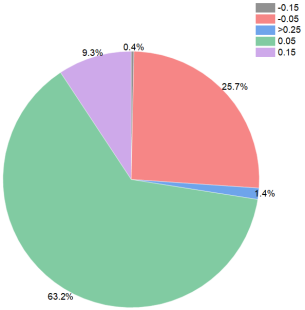}
				\caption{Accuracy increment of CN + LOPT}
				\label{xg_bing}
			\end{figure}

		\begin{table}[t]
\scriptsize
		\centering
		\begin{tabular}{l|r|r|r|r}
		\hline
		 & CN & CN+LOPT& BO & BCG \\
		\hline
		max	&1	&1	&1	&0.75	\\
		q3	&0.95		&0.95		&0.945	&0.58	\\
		median&0.73		&0.74		&0.73		&0.46	\\
		mean	&0.7279	&0.7371	&0.7097	&0.4678\\
		sd	&0.1833	&0.1791	&0.2061	&0.1510\\
		q1	&0.555	&0.57		&0.49		&0.39	\\
		min	&0.37		&0.41		&0.14		&0	\\
		\hline
		\end{tabular}
		\caption{Comparison of statistical characteristics.(XGBOOST)}
		\label{tabs:xg}
		\end{table}

			\begin{table}[t]
\scriptsize
			\centering
			\begin{tabular}{l|r|r|r}
			\hline
			 & CN & ZO& BCG \\
			\hline
			max	&0.975	&0.9625	&0.8625\\
			q3	&0.925	&0.8563	&0.4988\\
			median&0.835	&0.7763	&0.3963\\
			mean	&0.8353	&0.7680	&0.4322\\
			sd	&0.0889	&0.1136	&0.1978\\
			q1	&0.76		&0.6975	&0.2775\\
			min	&0.64		&0.4625	&0.0875\\
			\hline
			\end{tabular}
			\caption{Comparison of statistical characteristics.(CNN)}
			\label{tabs:cnn}
			\end{table}

 	\underline{CNN}
We test the effectiveness of our model (CN) on 180 classified datasets, 90 from mnist subsets and 90 from svhn subsets. The ZOOpt algorithm (ZO) was used as a control group, and a blank control group (BCG) without pre-training was also used to show the effectiveness of the proposed approach.
The optimal hyper-parameters generated by these three models are tested on the test set, and the accuracy is shown in Figure~\ref{cnn_acc}. The horizontal axis in the figure represents the test file, and the vertical axis represents the accuracy rate of CNN classification.\par
As with CNN, we analyze some statistics in Table~\ref{tabs:cnn}, and visualize them in Figure~\ref{cnn_acc_box}. First, according to the performance of BCG, we can observe that our dataset has a strong detection ability, which further illustrates that the CN model is effective in hyper-parameters prediction for CNN. Second, from the comparasion results of CN and ZO, our model outperforms ZO in various indicators, which shows that our model has strong generalization and migration capabilities.\par
In addition to the accuracy, we also compared the time overhead of CN and ZO, as shown in Figure~\ref{cnn_time}. The horizontal axis still represents different test files, and the vertical axis represents time overhead. The result is in log scale. From the  comparison results, our model CN outperforms the control group BO significantly. This coincides with the performance of CN on XGBOOST.

	\section{Related Work}
\label{sec:related}
The techniques of AutoML include model selection, automatic hyper-parameter optimization and automatic neural network structures construction. Here we focus on automatic hyper-parameter optimization.
Bayesian model is used to optimize hyper-parameters~\cite{FrazierA}. Ref.\cite{DBLP:journals/corr/abs-1805-04748} uses Bayesian algorithms to optimize the hyper-parameters of reinforcement learning. Even though they are effective, their efficiency prevent their applications on large datasets or the algorithm with many hyper-parameters. Although Some work solve the high-dimension problem~\cite{5c2348ceda562935fc1d576d,Ar}, the problem in efficiency is still not solved. Additionally, such approaches could be hardly transferred.\par
For the efficiency issue, some approaches have been proposed. 
An efficient automatic method~\cite{DBLP:journals/jise/LiHLLKT12} is proposed to optimize the parameters in the kernel of SVM by vectoring the kernel with sine/cosine algorithm~\cite{article}. However, this approach is too specific for SVM and fails to be generalized for other approaches.

Bandit-based hyper-parameters optimization\cite{bbb} accelerates the random search through adaptive resource allocation. 
Ref.~\cite{feurer2015initializing} considers extracting the meta-knowledge of datasets to calculate the hyper-parameters. But they neglect the relations between all datasets running on the same algorithm and the optimal hyper-parameters corresponding to these datasets. Rer.\cite{inproceedings} determines search direction through analysis of meta-knowledge. These approaches are orthogonal to ours and could be combined to our approach.

\section{Conclusion $\And$ Future Work}
\label{sec:con}
In this paper, we study hyper-parameter optimization for machine learning algorithms. We model the mapping from dataset to corresponding optimal hyper-parameters with neural network, obtaining the optimal hyper-parameter according to such relationship trained from generated datasets and their corresponding optimal hyper-parameters. To achive high-quality model, we design sophisticated network structure with effective training methods. With such model, the hyper-parameters could be derived  according to the dataset directly. To optimize the hyper-parameters furthermore, we also devleop local search strategies. Extensive experiments on real datasets shows that our approaches achieve high efficiency and effectiveness. In future work, we will design CN structures to optimize more algorithms. At the same time, we will design matching NPEs to adapt to video, audio, and text data, establishing a CN-NPE knowledge base.
\newpage
\bibliographystyle{named}
\bibliography{ref3}

\begin{thebibliography}{}

\bibitem[\protect\citeauthoryear{Barsce \bgroup \em et al.\egroup
  }{2018}]{DBLP:journals/corr/abs-1805-04748}
Juan~Cruz Barsce, Jorge~A. Palombarini, and Ernesto~C. Mart{\'{\i}}nez.
\newblock Towards autonomous reinforcement learning: Automatic setting of
  hyper-parameters using bayesian optimization.
\newblock {\em CoRR}, abs/1805.04748, 2018.

\bibitem[\protect\citeauthoryear{Fenwick}{1996}]{zkx_ref_16}
Peter Fenwick.
\newblock A new data structure for cumulative frequency tables.
\newblock {\em Software - Practice and Experience}, 24, 12 1996.

\bibitem[\protect\citeauthoryear{Feurer \bgroup \em et al.\egroup
  }{2015}]{feurer2015initializing}
Matthias Feurer, Jost~Tobias Springenberg, and Frank Hutter.
\newblock Initializing bayesian hyperparameter optimization via meta-learning.
\newblock pages 1128--1135, 2015.

\bibitem[\protect\citeauthoryear{Frazier}{2018a}]{FrazierA}
Peter Frazier.
\newblock {\em Bayesian Optimization}, pages 255--278.
\newblock 10 2018.

\bibitem[\protect\citeauthoryear{Frazier}{2018b}]{Frazier2018ATO}
Peter~I. Frazier.
\newblock A tutorial on bayesian optimization.
\newblock {\em ArXiv}, abs/1807.02811, 2018.

\bibitem[\protect\citeauthoryear{{Gao} \bgroup \em et al.\egroup
  }{2017}]{zkx_ref_4}
X.~{Gao}, S.~{Fan}, X.~{Li}, Z.~{Guo}, H.~{Zhang}, Y.~{Peng}, and X.~{Diao}.
\newblock An improved xgboost based on weighted column subsampling for object
  classification.
\newblock In {\em 2017 4th International Conference on Systems and Informatics
  (ICSAI)}, pages 1557--1562, Nov 2017.

\bibitem[\protect\citeauthoryear{Hamou \bgroup \em et al.\egroup
  }{2013}]{zkx_ref_1}
Reda Hamou, Abdelmalek Amine, and Ahmed Lokbani.
\newblock Study of sensitive parameters of pso application to clustering of
  texts.
\newblock {\em International Journal of Applied Evolutionary Computation},
  4:41--55, 04 2013.

\bibitem[\protect\citeauthoryear{Hochreiter and Schmidhuber}{1997}]{LSTM}
Sepp Hochreiter and Jürgen Schmidhuber.
\newblock Long short-term memory.
\newblock {\em Neural computation}, 9:1735--80, 12 1997.

\bibitem[\protect\citeauthoryear{Hu \bgroup \em et al.\egroup
  }{2018}]{inproceedings}
Yi-Qi Hu, Yang Yu, and Zhi-Hua Zhou.
\newblock Experienced optimization with reusable directional model for
  hyper-parameter search.
\newblock pages 2276--2282, 07 2018.

\bibitem[\protect\citeauthoryear{{Ichihashi} \bgroup \em et al.\egroup
  }{2011}]{zkx_ref_2}
H.~{Ichihashi}, K.~{Honda}, and A.~{Notsu}.
\newblock Comparison of scaling behavior between fuzzy c-means based classifier
  with many parameters and libsvm.
\newblock In {\em 2011 IEEE International Conference on Fuzzy Systems
  (FUZZ-IEEE 2011)}, pages 386--393, June 2011.

\bibitem[\protect\citeauthoryear{Ioffe and
  Szegedy}{2015}]{DBLP:journals/corr/IoffeS15}
Sergey Ioffe and Christian Szegedy.
\newblock Batch normalization: Accelerating deep network training by reducing
  internal covariate shift.
\newblock {\em CoRR}, abs/1502.03167, 2015.

\bibitem[\protect\citeauthoryear{Kanai \bgroup \em et al.\egroup
  }{2017}]{zkx_ref_12}
Sekitoshi Kanai, Yasuhiro Fujiwara, and Sotetsu Iwamura.
\newblock Preventing gradient explosions in gated recurrent units.
\newblock In I.~Guyon, U.~V. Luxburg, S.~Bengio, H.~Wallach, R.~Fergus,
  S.~Vishwanathan, and R.~Garnett, editors, {\em Advances in Neural Information
  Processing Systems 30}, pages 435--444. Curran Associates, Inc., 2017.

\bibitem[\protect\citeauthoryear{{LeCun} \bgroup \em et al.\egroup
  }{1989}]{6795724}
Y.~{LeCun}, B.~{Boser}, J.~S. {Denker}, D.~{Henderson}, R.~E. {Howard},
  W.~{Hubbard}, and L.~D. {Jackel}.
\newblock Backpropagation applied to handwritten zip code recognition.
\newblock {\em Neural Computation}, 1(4):541--551, Dec 1989.

\bibitem[\protect\citeauthoryear{Li \bgroup \em et al.\egroup
  }{2012}]{DBLP:journals/jise/LiHLLKT12}
Cheng{-}Hsuan Li, Hsin{-}Hua Ho, Yu{-}Lung Liu, Chin{-}Teng Lin, Bor{-}Chen
  Kuo, and Jin{-}Shiuh Taur.
\newblock An automatic method for selecting the parameter of the normalized
  kernel function to support vector machines.
\newblock {\em J. Inf. Sci. Eng.}, 28(1):1--15, 2012.

\bibitem[\protect\citeauthoryear{Li \bgroup \em et al.\egroup }{2017a}]{bbb}
Lisha Li, Kevin Jamieson, Giulia DeSalvo, Afshin Rostamizadeh, and Ameet
  Talwalkar.
\newblock Hyperband: A novel bandit-based approach to hyperparameter
  optimization.
\newblock {\em J. Mach. Learn. Res.}, 18(1):6765–6816, January 2017.

\bibitem[\protect\citeauthoryear{Li \bgroup \em et al.\egroup
  }{2017b}]{article}
Sai Li, Huajing Fang, and Xiaoyong Liu.
\newblock Parameter optimization of support vector regression based on sine
  cosine algorithm.
\newblock {\em Expert Systems with Applications}, 91, 08 2017.

\bibitem[\protect\citeauthoryear{Mutny and
  Krause}{2018}]{5c2348ceda562935fc1d576d}
Mojmir Mutny and Andreas Krause.
\newblock Efficient high dimensional bayesian optimization with additivity and
  quadrature fourier features.
\newblock {\em Annual Conference on Neural Information Processing Systems},
  pages 9019--9030, 2018.

\bibitem[\protect\citeauthoryear{Nwankpa \bgroup \em et al.\egroup
  }{2018}]{akak}
Chigozie Nwankpa, Winifred Ijomah, Anthony Gachagan, and Stephen Marshall.
\newblock Activation functions: Comparison of trends in practice and research
  for deep learning.
\newblock {\em CoRR}, abs/1811.03378, 2018.

\bibitem[\protect\citeauthoryear{Rolland \bgroup \em et al.\egroup }{2018}]{Ar}
Paul Rolland, Jonathan Scarlett, Ilija Bogunovic, and Volkan Cevher.
\newblock High-dimensional bayesian optimization via additive models with
  overlapping groups.
\newblock 02 2018.

\bibitem[\protect\citeauthoryear{Srivastava \bgroup \em et al.\egroup
  }{2014}]{ckck}
Nitish Srivastava, Geoffrey Hinton, Alex Krizhevsky, Ilya Sutskever, and Ruslan
  Salakhutdinov.
\newblock Dropout: A simple way to prevent neural networks from overfitting.
\newblock {\em Journal of Machine Learning Research}, 15:1929--1958, 06 2014.

\bibitem[\protect\citeauthoryear{{Tang} \bgroup \em et al.\egroup
  }{2019}]{zkx_ref_10}
H.~{Tang}, M.~{Lei}, Q.~{Gong}, and J.~{Wang}.
\newblock A bp neural network recommendation algorithm based on cloud model.
\newblock {\em IEEE Access}, 7:35898--35907, 2019.

\bibitem[\protect\citeauthoryear{{Trinchero} \bgroup \em et al.\egroup
  }{2019}]{zkx_ref_3}
R.~{Trinchero}, M.~{Larbi}, H.~M. {Torun}, F.~G. {Canavero}, and
  M.~{Swaminathan}.
\newblock Machine learning and uncertainty quantification for surrogate models
  of integrated devices with a large number of parameters.
\newblock {\em IEEE Access}, 7:4056--4066, 2019.

\bibitem[\protect\citeauthoryear{van Rijn and
  Hutter}{2018}]{Rijn2018HyperparameterIA}
Jan~N. van Rijn and Frank Hutter.
\newblock Hyperparameter importance across datasets.
\newblock {\em ArXiv}, abs/1710.04725, 2018.

\bibitem[\protect\citeauthoryear{Wang \bgroup \em et al.\egroup
  }{2017}]{zkx_ref_15}
Tiechao Wang, Shuai Sui, and Shaocheng Tong.
\newblock Data-based adaptive neural network optimal output feedback control
  for nonlinear systems with actuator saturation.
\newblock {\em Neurocomputing}, 247, 03 2017.

\bibitem[\protect\citeauthoryear{{Xu} \bgroup \em et al.\egroup
  }{2019}]{zkx_ref_9}
C.~{Xu}, Y.~{Dai}, R.~{Lin}, and S.~{Wang}.
\newblock Stacked autoencoder based weak supervision for social image
  understanding.
\newblock {\em IEEE Access}, 7:21777--21786, 2019.

\end{thebibliography}
\end{document}